\documentclass[12pt, final]{l4dc2023}

\title[Learning from CaTL+]{CatlNet: Learning Communication and Coordination Policies from CaTL+ Specifications}
\usepackage{times}
\usepackage{comment}
\usepackage{graphics}
\newtheorem{problem}{Problem}%[section]
\usepackage{enumitem}
\usepackage{diagbox}
\author{%
 \Name{Wenliang Liu} \Email{wliu97@bu.edu}\\
 \addr Boston University, Massachusetts, USA
 \AND
 \Name{Kevin Leahy} \Email{kevin.leahy@ll.mit.edu}\\
 \addr MIT Lincoln Laboratory, Lexington, MA, USA
 \AND
  \Name{Zachary Serlin} \Email{zachary.serlin@ll.mit.edu}\\
 \addr MIT Lincoln Laboratory, Lexington, MA, USA
 \AND
 \Name{Calin Belta} \Email{cbelta@bu.edu}\\
 \addr Boston University, Massachusetts, USA
}

\begin{document}

\maketitle

\begin{abstract}%
 In this paper, we propose a learning-based framework to simultaneously learn the communication and distributed control policies for a heterogeneous multi-agent system (MAS) under complex mission requirements from Capability Temporal Logic plus (CaTL+) specifications. Both policies are trained, implemented, and deployed using a novel neural network model called CatlNet. Taking advantage of the robustness measure of CaTL+, we train CatlNet centrally to maximize it where network parameters are shared among all agents, allowing CatlNet to scale to large teams easily. CatlNet can then be deployed distributedly. A plan repair algorithm is also introduced to guide CatlNet's training and improve both training efficiency and the overall performance of CatlNet. The CatlNet approach is tested in simulation and results show that, after training, CatlNet can steer the decentralized MAS system online to satisfy a CaTL+ specification with a high success rate. 
\end{abstract}


\section{Introduction}
\label{sec:intro}

Many real-world missions require the coordination of a heterogeneous Multi-Agent System (MAS). As collective tasks become increasingly complex, the need for an efficient way to define these tasks for a MAS becomes more and more stringent. This is especially true in cases where agents are controlled using learning-based methods (the focus of this paper), which may be advantageous for large MAS where coordination solutions are difficult to compute in real-time. Due to their expressivity and similarity to natural languages, temporal logics (such as Linear Temporal Logic (LTL) \cite{pnueli1977temporal} and Signal Temporal Logic (STL) \cite{maler2004monitoring}) have been widely used as specification languages for control systems. More recently, some work has focused on specifically tailoring temporal logics for MAS \cite{sahin2017provably,sahin2019multirobot,leahy2021scalable}. 
%The goal in these works focus on making it convenient to express cooperative task specifications for large teams of agents. 

%\zsinline{I think we could tie this paper together better if after the first paragraph we talk about using temporal logic for RL rewards and say this is a paper about encoding control and comms controllers as NNs using CaTL+ as the reward structure (that is generated automatically from CaTL+ specifications)}

In this paper, we focus on Capability Temporal Logic plus (CaTL+) \cite{liu2022robust}, which specifies rich task requirements with concrete temporal constraints for heterogeneous MAS. The agents can have different \emph{capabilities} of servicing \emph{tasks}. Besides qualitative semantics (whether requirements are satisfied), CaTL+ is also equipped with quantitative semantics, also called \emph{robustness}, which is a continuous real number that measures how strongly the requirements are satisfied. %This robustness signal is used in this work as a feedback signal to a learning agent. 
%\cbinline{we need to rewrite the following two paragraphs. it is not clear what the difference is from this paper and \cite{liu2022robust}. also not clear why we're making a comparison here since we have a section on related work} 
Taking advantage of this, controlling a MAS to satisfy a CaTL+ specification can be formulated as an optimization problem with the robustness as the objective function. In \cite{liu2022robust}, this problem was solved in one shot and results in an open-loop controller. However, computing this controller is time-consuming, and the open-loop controller is vulnerable under disturbances.  %To make the controller more robust under disturbances, we can close the loop by solving the problem in a Model Predictive Control (MPC) manner. However, this MPC controller is centralized and cannot be computed in real-time for large systems and complex specifications. \cbinline{it is not clear whether this was actually done before} 

In this paper we propose a learning-based framework to train a distributed control policy for each agent that collectively attempts to satisfy a given CaTL+ specification. By training the policy off-line, each agent can compute a feedback control in real-time. We assume that each agent can only observe its own state directly. However, satisfying a CaTL+ specification requires coordination of multiple agents, so communication is necessary. In practice, communication resources are usually limited, and how to utilize these resources is a challenging problem in itself. The framework in this paper jointly learns a  communication strategy (i.e., when and what each agent needs to communicate given limited bandwidth) together with a control policy. Under this communication  strategy, agents only communicate when necessary and transmit the most useful information. Both the control policy and communication strategy are implemented in a model that we call CatlNet, which consists of several neural networks (NNs). We train CatlNet with the centralized training and decentralized execution (CTDE) paradigm. The CatlNet parameters are shared for all agents in the training phase, so no additional parameters are needed when adding more agents, which makes the algorithm scalable for very large teams. 

%The main limitation of the CatlNet approach is not guaranteeing correctness (satisfaction), i.e., guarantees on whether the resulting team trajectory satisfies the given specification. Due to the distributed learning-based, neural network (NN), nature of CatlNet it is simply hard to make any strong correctness guarantees. 
Training the policy from scratch can be difficult especially when the task is complex. It has been shown in the literature (e.g., \cite{leung2022semi}) that expert demonstrations can help the optimizer converge. However, a dataset of expert demonstrations is not always available.  Hence, we design a repair scheme to fix the team trajectory generated by CatlNet such that it satisfies the CaTL+ specification. We use the repair algorithm to generate a dataset of satisfying trajectories to guide training, which is shown to improve the performance of the learned policies. 

The main contributions of this paper are twofold: (1) We propose a learning-based framework, called CatlNet, which can learn both the distributed control policy and the communication strategy given limited bandwidth to steer a MAS to satisfy a CaTL+ specification. (2) We designed a repair scheme to guide the training, which improves the performance of CatlNet. We show that the control and communication policies generated by this framework are reliable and computationally efficient.

\section{Related Work}
\label{sec:rw}

Controller synthesis from temporal logic specifications has gained significant attention in recent years. Roughly, existing approaches can be divided into two schools of thought: (1) Synthesis for LTL and fragments of LTL, which employ automata-based methods (see, e.g., \cite{belta2017formal}); (2) Synthesis for temporal logics defined over real-valued signals, such as STL, which can be formulated as optimization problems solved via Mixed Integer Programming (MIP) \cite{raman2014model}, \cite{sadraddini2015robust}
or gradient-based methods \cite{pant2017smooth}, 
\cite{gilpin2020smooth}. Both solution classes have also been extended to MAS. The authors in \cite{chen2011synthesis, schillinger2018simultaneous, kantaros2020stylus, luo2021abstraction} applied automata-based methods to synthesize distributed control policies from a global LTL specification. Logics specifically designed for MAS including counting LTL (cLTL) \cite{sahin2017provably}, cLTL+ \cite{sahin2019multirobot}, Capability Temporal Logic (CaTL) \cite{leahy2021scalable} and STL with integral predicates \cite{buyukkocak2021planning} have also been proposed; MIP is used for control synthesis. An extension of CaTL, called CaTL+, was proposed in \cite{liu2022robust}. Taking advantage of differentiable robustness, control synthesis from CaTL+ is solved using gradient-based methods. 
All the methods mentioned above either synthesize the control in one shot or compute the control online. Hence, they are computationally very expensive for large MAS, which prohibits their use for real-time control.

Learning-based methods can be used to move online computation offline, which enables real-time executions for the above methods. Reinforcement Learning (RL) was combined with automata-based methods \cite{li2019formal,cai2021modular} and optimization-based methods \cite{aksaray2016q,liu2021model} to synthesize control policies for a single agent systems under temporal logic specifications. Model-free RL has also been applied to MAS under LTL \cite{sun2020automata,hammond2021multi,zhang2022distributed} and STL \cite{muniraj2018enforcing} specifications. However, model-free RL requires a large number of trials to learn the policy, which might be infeasible in practice. In this paper, we apply model-based RL and assume the system model is known. Though not included in this work, the model can also be learned from data. This paper can be seen as an extension of \cite{liu2021model}, moving from a single agent system and STL to MAS and CaTL+. 

Improving the performance of NNs is investigated in \cite{ma2020stlnet} where a method called STLnet is proposed to project a sequence of NN outputs to satisfy an STL formula. However, STLnet is not designed for control systems and it cannot fix the controller given STL over states. Inspired by STLnet though, we repair the controls given by CatlNet to guide the training. %such that the state trajectory satisfies the CaTL+ formula. 

Finally, this work can also be viewed in the context of multi-agent RL and distributed networks in which communication is learned. Due to the importance of communication in cooperative tasks, many RL frameworks that can simultaneously learn control and communication policies were proposed recently, such as DIAL \cite{foerster2016learning}, CommNet \cite{sukhbaatar2016learning}, BicNet \cite{peng2017multiagent} and ATOC \cite{jiang2018learning}. The most related framework to ours is ATOC, which is designed for homogeneous MAS with a given reward function. The communication architecture of our CatlNet is inspired from ATOC, and extends it for heterogeneous teams under CaTL+ specifications. Using CaTLNet, rewards are generated automatically from the CaTL+ formula.

\section{Preliminaries}
\label{sec:pre}
\subsection{System Model}

We use bold and calligraphic symbols to represent trajectories and sets, respectively. $|\mathcal{X}|$ is the cardinality of a set $\mathcal{X}$. Consider a team of agents labelled from a finite set $\mathcal{J}$, where $j \in \mathcal{J}$ denotes an agent's index. We assume that all agents share the same state space $\mathcal X\subseteq\mathbb{R}^{n_{x}}$ and discrete time dynamics (a relaxation to this will be discussed in Remark~\ref{rm:dyn}):
\vspace{-2pt}
\begin{equation}
    \label{eq:system}
    x_{j}(t+1) = x_{j}(t) + u_{j}(t),\quad t=0,1,\ldots,H-1,
    \vspace{-2pt}
\end{equation}
where $x_{j}(t)\in\mathcal X$ and $u_{j}(t)\in\mathcal U_j\subset\mathbb R^{n_u}$ are the state and control at time $t$, $\mathcal U_j$ is the control space of agent $j$, and $H$ is a finite time horizon determined by the mission specification. Each agent $j$ is assumed to have a random initial state in $\mathcal X_{j,0}\subset\mathcal{X}$. Let $P_j: \mathcal X_{j,0} \rightarrow \mathbb R$ be the probability density function of the initial state $x_{j}(0)$. Consider a finite set of capabilities $Cap$ for team $\mathcal{J}$. Each agent has its own set of capabilities $Cap_j\subseteq Cap$. We assume that $\cup_{j\in \mathcal{J}}Cap_j= Cap$. 

The trajectory of an agent $j$, called an \emph{individual trajectory}, is a sequence $\mathbf{x}_j=x_{j}(0) \ldots x_{j}(H)$.  Then \emph{team trajectory} is defined as a set of pairs $\mathbf X = \{(\mathbf x_j,Cap_j)\}_{j\in\mathcal J}$, which captures all the \emph{individual trajectories} with their corresponding capabilities. Here we include capabilities in a team trajectory so that we can define the semantics of CaTL+ on it, as it will be shown later. Let $\mathcal J_c = \{j\;|\;c\in Cap_j\}$ be the set of agent indices with capability $c$. Let $\mathbf u_j = u_{j}(0)\ldots u_{j}(H-1)$ be the sequence of controls for agent $j$, $\bar x(t) = [x_j(t)]_{j=1}^{|\mathcal J|}$ and $\bar u(t) = [u_j(t)]_{j=1}^{|\mathcal J|}$ be the joint state and control of the MAS at time $t$. Denote $\mathbf x_{0:t}^j =x_{j}(0) \ldots x_{j}(t)$ and $\mathbf{\bar{x}}_{0:t} = \bar x(0), \ldots, \bar x(t)$.

\begin{remark}
\label{rm:dyn}
We simplify each agent's dynamics to a single integrator as in \eqref{eq:system} such that all agents can be controlled by CatlNet (described in Sec.~\ref{sec:catlnet}) with same parameters. In fact, these dynamics can be seen as high level nominal dynamics used to generate a sequence of waypoints. The true dense-time dynamics of the agents can be heterogeneous, as long as all agents share a common workspace $\mathcal X$ (not necessarily the state space). By properly selecting the control constraint $\mathcal U_j$, we can find a local controller for each agent that tracks the individual trajectory (waypoints) within given time window and avoids inter-agent collision, using the techniques in \cite{sun2022multi}.
\end{remark}

\subsection{CaTL+ Syntax and Semantics}

Capability Temporal Logic plus (CaTL+) \cite{liu2022robust} is a two-layer logic that includes an inner logic and outer logic. The \emph{inner} logic defined over individual trajectories $\mathbf x$ (subscript $j$ omitted for simplicity) is identical to STL and has the following syntax:
\vspace{-2pt}
\begin{equation}
\label{eq:syntax-in}
\varphi:=True\;|\;\mu \; | \; \neg\varphi \; | \; \varphi_1\land\varphi_2 \; | \; \varphi_1\lor\varphi_2 \;|  \;  \varphi_1 \mathbf{U}_{[a,b]} \varphi_2,
\vspace{-2pt}
\end{equation}
where $\varphi$, $\varphi_1$ and $\varphi_2$ are inner logic formulas, $\mu$ is a \emph{predicate} in the form of $f(x(t))\geq 0$. We assume $f:\mathcal X\rightarrow \mathbb R$ is a differentiable function. $\neg$, $\land$, $\lor$ are the Boolean \emph{not}, \emph{conjunction} and \emph{disjunction} respectively. $\mathbf{U}_{[a,b]}$ is the temporal operator \emph{until}, where $\varphi_1 \mathbf{U}_{[a,b]} \varphi_2$ means ``$\varphi_2$ must become true at some time point in $[a,b]$ and $\varphi_1$ must stay true before that", $[a,b]$ are all integer time points between $a$ and $b$. Other temporal operators like \emph{eventually} $\mathbf F_{[a,b]}\varphi$ and \emph{always} $\mathbf G_{[a,b]}\varphi$ are defined as $\mathbf F_{[a,b]}\varphi = True\mathbf{U}_{[a,b]}\varphi$ and $\mathbf G_{[a,b]}\varphi = \neg\mathbf F_{[a,b]}\neg\varphi$, where $\mathbf F_{[a,b]}\varphi$ states that ``$\varphi$ becomes true at some time point in $[a,b]$" and $\mathbf G_{[a,b]}\varphi$ states that ``$\varphi$ stays true at all time points in $[a,b]$". An individual trajectory $\mathbf x$ satisfies a inner logic (STL) $\varphi$ at time $t$ is denoted as $(\mathbf x,t)\models \varphi$. 

The outer logic (with a slight abuse of terminology we refer it as CaTL+), which is defined over team trajectories, has similar syntax with STL, except for predicates $\mu$ are replaced by \emph{tasks} $T$:
\vspace{-2pt}
\begin{equation}
\label{eq:syntax-out}
\Phi:=True\;|\;T \; | \; \neg\Phi \; | \; \Phi_1\land\Phi_2 \; | \; \Phi_1\lor\Phi_2 \;|  \;  \Phi_1 \mathbf{U}_{[a,b]} \Phi_2,
\vspace{-2pt}
\end{equation}
where $\Phi$, $\Phi_1$ and $\Phi_2$ are CaTL+ formulas, $T=\langle\varphi,c,m\rangle$ is a \emph{task}, $\varphi$ is an inner logic formula, $c\in Cap$ is a capability, and $m$ is a positive integer. The other operators are the same as the ones in STL. A task is satisfied at time $t$ if and only if at least $m$ \emph{individual trajectories} of agents with capability $c$ satisfy $\varphi$ at time $t$. Formally, we define a counting function $n(\mathbf X,c,\varphi,t)$ to capture this:
\vspace{-2pt}
\begin{equation}
    n(\mathbf X,c,\varphi,t) = \sum_{j\in\mathcal J_c}I\big( (\mathbf x_j,t)\models \varphi \big),
    \vspace{-2pt}
\end{equation}
where $I$ is an indicator function, i.e., $I = 1$ if $(\mathbf x_j,t)\models \varphi$ and $I = 0$ otherwise. Then the team trajectory $\mathbf X$ satisfies $T$ at time $t$, denoted by $(\mathbf X, t)\models T$, if and only if $n(\mathbf X,c,\varphi,t) \geq m$.

CaTL+ not only has qualitative semantics, i.e., \emph{whether} $\mathbf X$ satisfies $\Phi$, but also has quantitative semantics (also called robustness), i.e., \emph{how much} the specification is satisfied or violated. Denote the (exponential) robustness of a CaTL+ formula $\Phi$ with respect to a team trajectory $\mathbf X$ at time $t$ as $\eta(\mathbf X,\Phi,t)$, which is differentiable almost everywhere. The detailed definition of $\eta$ can be found in \cite{liu2022robust}. The robustness of CaTL+ is sound, i.e., $\eta(\mathbf X,\Phi,t)\geq0$ if and only if $(\mathbf X, t)\models \Phi$. The time horizon of a CaTL+ formula $\Phi$, denoted by $hrz(\Phi)$, is defined as the closest future time point that is needed to decide the satisfaction of $\Phi$.

\begin{example}
\label{ex}
To provide a comparison, we use the earthquake emergency response scenario defined in \cite{liu2022robust}. The workspace $\mathcal X \subset \mathbb R^2$ is shown in Fig.~\ref{fig:results}(a). There are $4$ ground vehicles $j\in\{1,2,3,4\}$ and $2$ aerial vehicles $j\in\{5,6\}$, totaling $6$ robots indexed from $\mathcal{J} = \{1,2,3,4,5,6\}$. A bridge $B$ goes across a river $R$ in the area. All ground vehicles start from initial state $x_j(0)$ uniformly sampled in region $Init_g$ and have capabilities $Cap_j = \{``Delivery", ``Ground"\}$, $j\in\{1,2,3,4\}$. All the aerial vehicles have
initial state $x_j(0)$ uniformly sampled in the region $Init_a$ and have capabilities $Cap_j = \{``Delivery", ``Inspection"\}$, $j\in\{5,6\}$. 

Consider the following CaTL+ specifications: (1) $\Phi_1=\langle \mathbf F_{[0,8]} x\in C,\ ``Delivery",\ 6 \rangle$: $6$ agents with capability $``Delivery"$ should pick up supplies from region $C$ within $8$ time units; (2) $\Phi_2=\langle \mathbf F_{[0,25]} x\in V_1,\ ``Delivery",\ 3 \rangle \land \langle \mathbf F_{[0,25]} x\in V_2,\ ``Delivery",\ 3 \rangle$: $3$ agents with capability $``Delivery"$ should deliver supplies to the affected village $V_1$ and $V_2$ within $25$ time units, respectively; (3) $\Phi_3=\neg\langle x\in B,\ ``Ground",\ 1 \rangle \mathbf  U_{[0,5]}\langle x\in B,\ ``Inspection",\ 2 \rangle$: any agent with capability $``Ground"$ cannot go over the bridge until $2$ agents with capability $``Inspection"$ inspect it within $5$ time units; (4) $\Phi_4=\mathbf G_{[0,25]}\langle \neg(x\in R),\ ``Ground",\ 4 \rangle$: agents with capability $``Ground"$ should always avoid entering the river $R$; (5) $\Phi_5=\mathbf G_{[0,25]}\neg\langle x\in B,\ ``Ground",\ 2 \rangle$: Since the load of the bridge is limited, at all times no more than $1$ agent with capability $``Ground"$ can be on $B$; (6) $\Phi_6=\mathbf G_{[0,25]}\langle x\in M,\ ``Delivery",\ 6 \rangle$: $6$ agents with capability $``Delivery"$ should always stay in region $M$.
\begin{comment}
\begin{itemize}[leftmargin=*]
\setlength\itemsep{-2pt}
    \item  $\Phi_1=\langle \mathbf F_{[0,8]} x\in C,\ ``Delivery",\ 6 \rangle$: $6$ agents with capability $``Delivery"$ should pick up supplies from region $C$ within $8$ time units; 
    \item $\Phi_2=\langle \mathbf F_{[0,25]} x\in V_1,\ ``Delivery",\ 3 \rangle \land \langle \mathbf F_{[0,25]} x\in V_2,\ ``Delivery",\ 3 \rangle$: $3$ agents with capability $``Delivery"$ should deliver supplies to the affected village $V_1$ and $V_2$ within $25$ time units, respectively;
    \item $\Phi_3=\neg\langle x\in B,\ ``Ground",\ 1 \rangle \mathbf  U_{[0,5]}\langle x\in B,\ ``Inspection",\ 2 \rangle$: the bridge might be affected by the earthquake so any agent with capability $``Ground"$ cannot go over it until $2$ agents with capability $``Inspection"$ inspect it within $5$ time units; 
    \item $\Phi_4=\mathbf G_{[0,25]}\langle \neg(x\in R),\ ``Ground",\ 4 \rangle$: agents with capability $``Ground"$ should always avoid entering the river $R$;
    \item $\Phi_5=\mathbf G_{[0,25]}\neg\langle x\in B,\ ``Ground",\ 2 \rangle$: Since the load of the bridge is limited, at all times no more than $1$ agent with capability $``Ground"$ can be on $B$;
    \item $\Phi_6=\mathbf G_{[0,25]}\langle x\in M,\ ``Delivery",\ 6 \rangle$: $6$ agents with capability $``Delivery"$ should always stay in region $M$.
\end{itemize} 
\end{comment}
The overall specification for the system is $\Phi = \bigwedge_{i=1}^6 \Phi_i$, with $hrz(\Phi)=25$. An example team trajectory is shown in Fig.~\ref{fig:map}, where $\Phi_1$ is satisfied because all $6$ agents enter $C$ while $\Phi_4$ is violated since a ground vehicle falls into $R$. 
\end{example}

\section{Problem Formulation and Approach}

Consider a team of agents $\mathcal J$ that needs to collaboratively satisfy a CaTL+ specification $\Phi$. We assume that: (1) each agent can only observe its own state $x_j(t)$ at each time $t$; (2) all agents have access to a communication channel for all times. At each time $t$, each agent can broadcast a vector $h_t^j\in\mathbb R^{n_c}$ to the channel and receive a vector $\tilde{h}_t^j\in\mathbb R^{n_c}$ from the channel. The dimension of the communication vectors are fixed because of the limited bandwidth of the channel. We formulate the joint control and communication synthesis problem as:
\begin{problem}
\label{pb:1}
Given a multi-agent system $\mathcal J$ with initial states $\{x_j(0)\in\mathcal X_{0,j}\ |\ j\in \mathcal J \}$ distributed as $P_j$
and a CaTL+ specification $\Phi$ defined over the team trajectory $\mathbf X$, find the control policy $u_j(t) = \pi_j(\mathbf x_{0:t}^j,h_t^j,\tilde h_t^j)$ and the communication vectors $h_t^j$ and $\tilde h_t^j$ that maximize the objective:\vspace{-2pt}
\begin{equation}
    \label{eq:pb1}
    \begin{aligned}
    \max_{\pi_j,h_t^j,\tilde h_t^j,\ j\in\mathcal J}\quad &\eta(\mathbf{X},\Phi,0) - \gamma \cdot \max\big(\eta(\mathbf{X},\Phi,0),0\big) \cdot\sum_{j\in\mathcal J}\sum_{t=0}^{H-1}C\big(\pi_j(\mathbf x_{0:t}^j,h_t^j,\tilde h_t^j)\big) \\
    \text{s.t.}\quad & x_{j}(t+1) = f_j(x_{j}(t),\pi_{j}(\mathbf x_{0:t}^j,h_t^j,\tilde h_t^j)),\\ &\pi_{j}(\mathbf x_{0:t}^j,h_t^j,\tilde h_t^j)\in\mathcal U_j,\ t=0,\ldots,H-1,
    \end{aligned}    
    \vspace{-2pt}
\end{equation}
where $\eta(\mathbf X,\Phi,0)$ is the CaTL+ robustness, $C(\cdot)$ is a cost function, $H\geq hrz(\Phi)$ is the planning horizon, and $\gamma$ is a parameter satisfying $\gamma\geq\sup_{u_j(t)\in\mathcal U_j}\sum_{j\in\mathcal J}\sum_{t=0}^{H-1}C(u_j(t))$.

In addition, our secondary objective is to minimize the total number of times that agents access the communication channel to save the energy cost on communication.
\end{problem}

Since the satisfaction of $\Phi$ is always a priority, the constraint on $\gamma$ ensures that the objective function has the same sign as $\eta(\mathbf X,\Phi,0)$, so minimizing the cost never overrides maximizing the robustness. Note that to determine an agent's control at time $t$, history states of the agent are needed due to the temporal requirements (as described in \cite{liu2021recurrent}).

A straightforward method to get $\pi_j$ in Pb.~\ref{pb:1} is to apply Model Predictive Control (MPC), i.e., compute a sequence of controls within a planning horizon and apply the first one to the system at each time. Since the objective function is defined on the entire team, the history information of all agents is required by the MPC controller at all time. This information needs to be either stored at a central node (taking up a large storage space) where $h_t^j=x_j(t)$, or sent from each agent to the communication channel (i.e., $h_t^j=\mathbf x_{0:t}^j$, which results in a large communication vector). Moreover, since the robustness of CaTL+ is nonconvex, solving the MPC problem at each time step can be time-consuming. Finally, MPC requires all agents to connect to the channel at all time steps, which is not going to accomplish our secondary objective. 
Hence, this MPC approach can be intractable in practice due to the limits on communication, storage and the real-time computation requirement. 

In this paper, we propose a learning-based algorithm (detailed in Sec.~\ref{sec:catlnet}) to solve Pb.~\ref{pb:1}. The algorithm finds the communication strategy, i.e., whether an agent needs to communicate at time $t$ and what the communication vectors $h_t^j$ and $\tilde{h}_t^j$ are, and a distributed control policy for each agent that can compute the control in real time. We propose a NN-based framework, called CatlNet, to implement both policies. CTDE paradigm is applied, i.e., we assume all state information is known during training. During execution however, each agent only observe its own state but it is connected to a communication channel as mentioned at the beginning of this section. After training, CatlNet can be generalized to random initial states $x_j(0)$ with known initial distribution $P_j$.

\vspace{-3pt}
\section{CatlNet}
\label{sec:catlnet}
\vspace{-3pt}
\subsection{Architecture of CatlNet}
\vspace{-3pt}

\begin{figure}
\centering
\begin{minipage}[t]{0.47\textwidth}
    \centering
    \includegraphics[height=5.4cm]{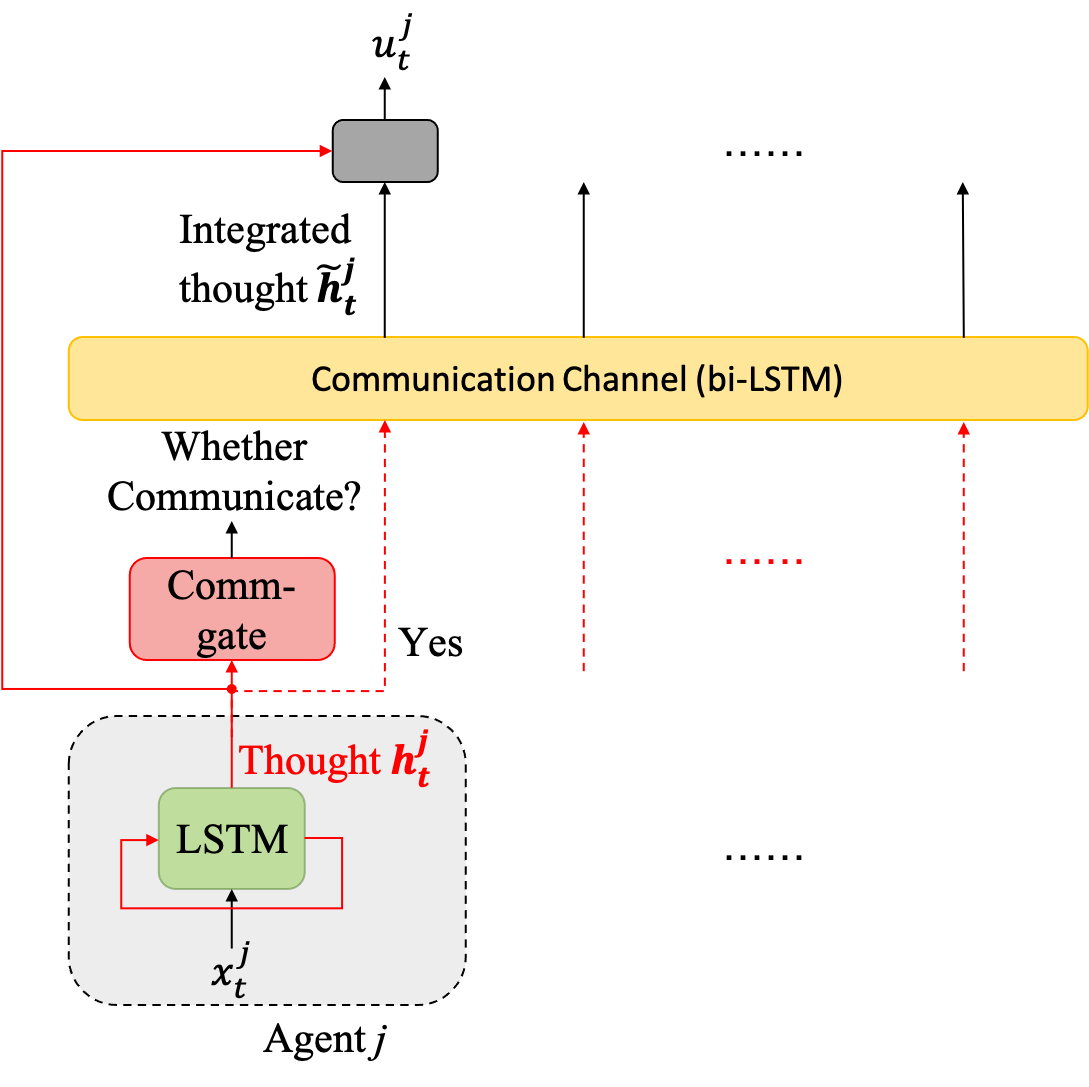}
    \caption{\vspace{-4pt}\small The overall architecture of CatlNet.}
    \label{fig:catlnet}
\end{minipage}
\quad
\begin{minipage}[t]{0.47\textwidth}
    \centering
    \includegraphics[height=5.4cm]{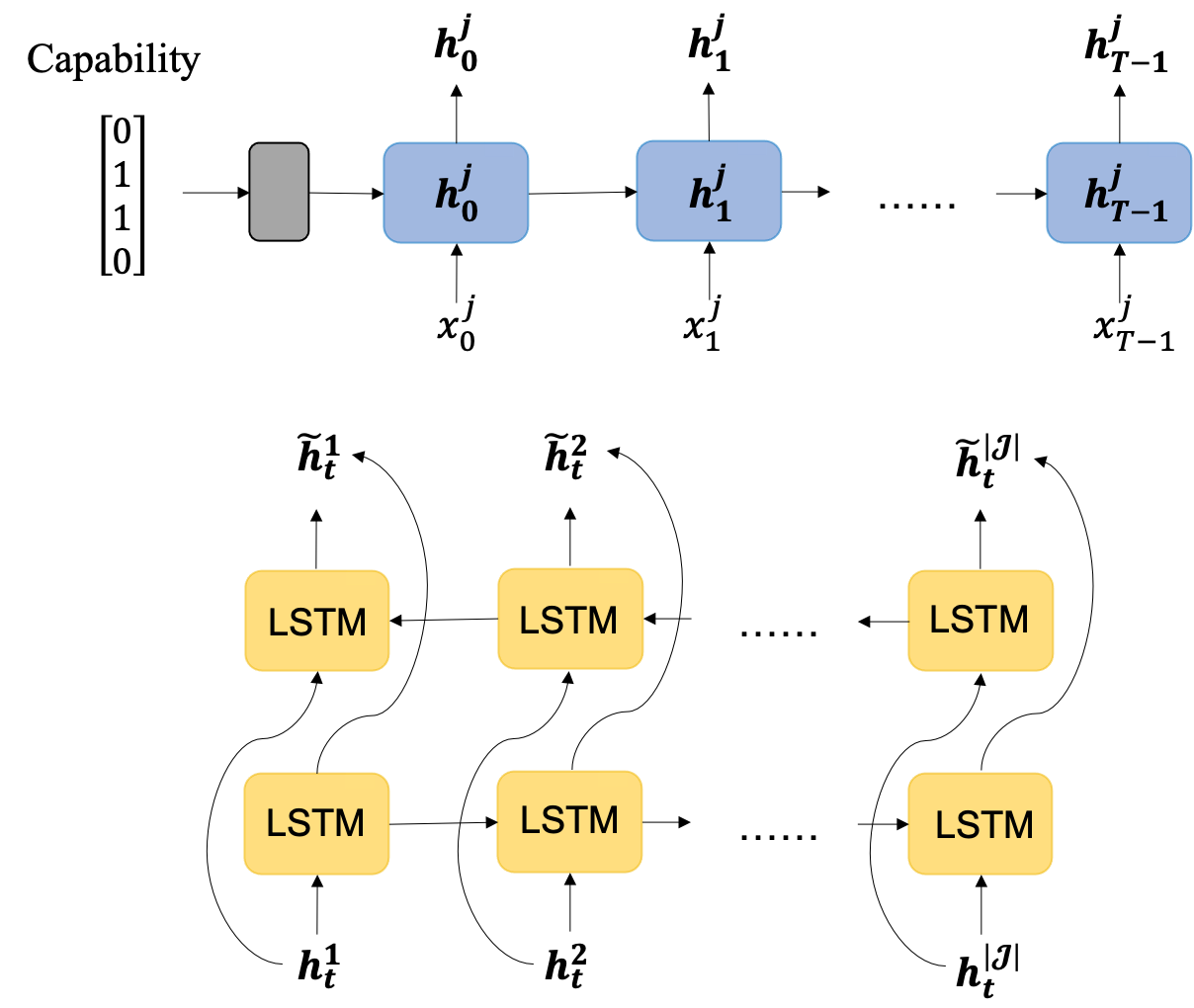}
    \caption{\vspace{-4pt}\small CapNet and unfolded LSTM (above), bi-directional LSTM (bottom).}
    \label{fig:lstm}
\end{minipage}
\vspace{-21pt}
\end{figure}

The overall architecture of CatlNet is shown in Fig.~\ref{fig:catlnet}. We extract a vector $h_t^j$ called \emph{thought} from the state of agent $j$ at time $t$ using a Long Short Term Memory (LSTM) NN with parameters $\theta_L$ (Fig~\ref{fig:lstm} above). Here the thought contains the information of history states. Let $Cap_j^v$ be the vectorized representation of agent $j$'s capability, where $Cap_j^v=[b_1\ b_2\cdots b_{|Cap|}]^\top$, $b_i=1$ if $c_i\in Cap_j$, $b_i=0$ if $c_i\not\in Cap_j$. We input $Cap_j^v$ to a NN called CapNet with parameters $\theta_{cap}$ and use its output as the initial hidden state of the LSTM as shown in Fig.~\ref{fig:lstm} (above). The output of the LSTM, i.e., the thought $h_t^j$ is then fed to a classifier, called a Comm-gate, to decide whether the agent communicates. The Comm-gate is also implemented as a NN with parameters $\theta_{g}$. If the Comm-gate decides to communicate, then the thought $h_t^j$ is passed to a communication channel implemented by a bi-directional LSTM with parameters $\theta_b$. The communication channel can merge all agents' thoughts (who decide to communicate at that time) and output an integrated thought $\tilde{h}_t^j$ that guides agents to generate coordinated actions as shown in Fig.~\ref{fig:lstm} (bottom). Then the integrated thought $\tilde{h}_t^j$ is sent back to agent $j$ and concatenated with its original thought $h_t^j$. Finally, $[h_t^j,\tilde{h}_t^j]$ is fed to another NN called OutNet with parameters $\theta_o$, and it outputs the control $u_j(t)$. A hyperbolic tangent function is applied at the last layer of OutNet to satisfy the constraint $u_j\in\mathcal U_j$ as in \cite{yaghoubi2019worst}. We denote the control policy given by CatlNet as $\bar u(t) = \bar\pi(\mathbf{\bar x}_{0:t},\theta_{cap},\theta_L,\theta_b,\theta_o,\theta_g)$.

Note that in this framework all agents share the same NN parameters ( $\theta_{cap}$, $\theta_L$, $\theta_g$, $\theta_o$), so the number of trainable parameters does not increase as the number of agents increases. Capabilities fed as initial hidden state to the LSTM and communication make agents behave differently.

\subsection{Training of CatlNet}
\label{sec:train}
Following the CTDE paradigm, CatlNet is trained as two parts: (1) the policy networks, including CapNet, LSTM, communication channel and OutNet; and (2) the Comm-gate. 
\vspace{-2pt}
\subsubsection{Training of the Policy Networks}
We initially ignore the Comm-gate network training and let all agents communicate openly to the channel, denoted as $\theta_g=\theta_g^{full}$. Since the control policy has been parameterized by CatlNet and we want to generalize CatlNet to different initial states, the first objective in Pb.~\ref{pb:1} becomes:
\begin{problem}
Given a multi-agent system $\{A_j\ |\ j\in \mathcal J \}$ and a CaTL+ specification $\Phi$ defined over the team trajectory $\mathbf X$, find the optimal CatlNet parameters $\theta_{cap}$, $\theta_L$, $\theta_b$, $\theta_o$ that maximizes the objective:
\vspace{-4pt}
\begin{equation}
    \label{eq:training}
    \begin{aligned}
    \max_{\theta_{cap},\theta_L,\theta_b,\theta_o}\quad &E_{P(x_j(0))}\big[\eta(\mathbf{X},\Phi,0) - \gamma \cdot \max\big(\eta(\mathbf{X},\Phi,0),0\big) \cdot\sum_{j\in\mathcal J}C(\mathbf u_j)\big] \\
    \text{s.t.}\quad &\bar u(t)=\bar\pi(\mathbf{\bar x}_{0:t},\theta_{cap},\theta_L,\theta_b,\theta_o,\theta_g^{full}),\\
    & x_{j}(t+1) = f_j(x_{j}(t),u_{j}(t)),\ t=0,\ldots,H-1.
    \end{aligned}    
    \vspace{-2pt}
\end{equation}
\end{problem}
In practice, we randomly sample $M$ initial states of the MAS and use the average to approximate the expectation in \eqref{eq:training}. We substitute the constraints to the objective function to make \eqref{eq:training} an unconstrained optimization problem and use the Adam stochastic optimizer \cite{kingma2014adam} to update $\theta_{cap}$, $\theta_L$, $\theta_b$, $\theta_o$. We resample $M$ initial states at each optimization step. Note that all gradients can be computed automatically and analytically using the technique in \cite{leung2020back}.

When the CaTL+ specification is complex, the policy may become stuck in a local optima that violates the specification. To improve training reliability, we consider how humans learn. 
(1) Given an objective, a human learner might achieve it with a suboptimal solution. On the other hand, if optimal demonstrations are provided, a learner who do not know the objective can imitate but might fail given unseen conditions. Hence, the best strategy is to provide the learner both the objective and the demonstrations. (2) Facing a bunch of unfamiliar demonstrations, it might be hard for a learner to discover the underlying rules. However, if a coach shows the learner a solution which is an adaptation of the learner's own behavior, it will be easier for the learner to improve her solution. Inspired by these two insights, we designed a repair algorithm (Alg.~\ref{alg:repair} described in Sec.~\ref{sec:repair}), which can fix the team trajectory generated by CatlNet to satisfy the CaTL+ specification. We use the repair algorithm to guide the training and further improve the policy. 

We use the trained CatlNet to generate a set of $N$ team trajectories starting from random initial states and collect the violating trajectories. Then we use Alg.~\ref{alg:repair} to repair them and collect all successfully repaired team trajectories to form a dataset $D=\{\mathbf X_{d}^{(i)}|i=1,\ldots, N\}$. 
Denote the objective in \eqref{eq:training} as $L(\theta_{cap},\theta_L,\theta_b,\theta_o,\theta_g)$. Then train CatlNet again to maximize the objective:
\vspace{-4pt}
\begin{equation}
    \label{eq:semi}
    \begin{aligned}
    \max_{\theta_{cap},\theta_L,\theta_b,\theta_o}\quad & (1-\beta)L(\theta_{cap},\theta_L,\theta_b,\theta_o,\theta_g^{full}) - \beta \sum_{i=1}^{N}\sum_{t=0}^{T-1} \| \bar x^{(i)}(t) - \bar x^{(i)}_d(t)\|^2 \\
    \text{s.t.}\quad & x_{j}^{(i)}(t+1) = f_j(x_{j}^{(i)}(t),u_{j}^{(i)}(t)),\quad x_j^{(i)}(0) = x_{j,D}^{(i)}(0),\quad i=1,\ldots,N,\\ 
    &\bar u^{(i)}(t)=\bar\pi(\mathbf {\bar x}_{0:t}^{(i)},\theta_{cap},\theta_L,\theta_b,\theta_o,\theta_g^{full}),\ t=0,\ldots,H-1,
    \end{aligned}    
    \vspace{-1pt}
\end{equation}
where $\beta\in[0,1]$ balances maximizing \eqref{eq:training} and imitating the dataset. Note that we can rearrange the identical agents in the team to minimize $\| \bar x^{(i)}(t) - \bar x^{(i)}_d(t)\|^2$ in \eqref{eq:semi}, which makes the dataset permutation-invariant. We keep repairing the violated trajectories, adding them to dataset $D$ and retrain CatlNet until convergence. Now we have obtained a control policy with full communication.

\vspace{-4pt}
\subsubsection{Training of the Comm-gate}
To train the Comm-gate that decides when an agent needs to communicate, we first use the objective \eqref{eq:semi} with the final dataset to train another CatlNet with no communication at all (likely it cannot satisfy the specification). Then we generate trajectories using the full communication CatlNet, but disable the communication channel connection for one agent $j$ at one time point $t$ at a time. The chosen agent $j$ will use the no communication CatlNet instead at the chosen time $t$. Then we compare the robustness values with and without this deactivation. If the deactivation makes the robustness decrease over a threshold, we label the thought $h_t^{j,(i)}$ with $y^{(i)}=1$ (the agent should communicate), otherwise we label $h_t^{j,(i)}$ with $y^{(i)}=0$ (the agent does not need communication). Repeat this from sampled initial states to form a dataset $D_g$ consist of data pairs $(h_t^{j,(i)},y^{(i)})$ that covers all agents at all time points. Let $G(h_t^j,\theta_g)\in\mathbb R^2$ be the output of Comm-gate. We train the Comm-gate on $D_g$ as a standard classifier to minimize the cross entropy loss:
\vspace{-3pt}
\begin{equation}
    \label{eq:class}
    \min_{\theta_g}\ \sum_{D_g} -y^{(i)}\log\Big( \sigma\big(G(h_t^{j,(i)},\theta_g)\big)_1\Big) - (1-y^{(i)})\log\Big(\sigma\big(G(h_t^{j,(i)},\theta_g)\big)_2\Big)
    \vspace{-4pt}
\end{equation}
where $\sigma:\mathbb R^2\rightarrow\mathbb R^2$ is the softmax function, $\sigma()_1$ and $\sigma()_2$ denote its 1st and 2nd elements. Finally, we train CatlNet with the Comm-gate using the objective \eqref{eq:semi}, which gives us the final CatlNet. 

\vspace{-3pt}
\subsection{Repair of CatlNet}
\label{sec:repair}
Now we describe the repair algorithm used in the training phase. The first step is to rewrite the outer logic of CaTL+ into a negation-free Disjunctive Normal Form (DNF, a disjunction of conjunctions). To do this, we extend the definition of the CaTL+ task. Define a \emph{timed task} as $\bar T=\langle\varphi,c,m\rangle_t$, such that $(\mathbf X, t_0) \models \bar T$ iff $(\mathbf X, t_0+t) \models \langle\varphi,c,m\rangle$. That is, a timed task $\bar T=\langle\varphi,c,m\rangle_t$ is required to be satisfied at time $t$. The original task $\langle\varphi,c,m\rangle$ is equivalent to $\langle\varphi,c,m\rangle_{t=0}$.

\vspace{-5pt}
\begin{proposition}
Every CaTL+ formula can be represented in the negation-free DNF: $ \bigvee_{k=1}^K\bigwedge_{i=1}^{I_k}\bar T_{i}^k$ where $\bar T_{i}^k$ is a timed task, $I_k$ is the number of conjunctions in the $k^{th}$ disjunction.
\end{proposition}

\begin{proof}
[Sketch] Since the syntax of CaTL+ outer logic is identical with STL except that predicates are replaced by \emph{tasks}, we can follow Algorithm 1 in \cite{ma2020stlnet} to represent any CaTL+ formula into DNF form, where negations are only applied to (timed) tasks. For these negative tasks, we have: $(\mathbf X,t_0)\models\neg\langle\varphi,c,m\rangle_t$ is equivalent to $(\mathbf X,t_0)\models\langle\neg\varphi,c,|\mathcal J_c|-m\rangle_t$. In other words, ``no more than $m-1$ agents with capability $c$ satisfy $\varphi$" is equivalent to ``at least $|\mathcal J_c|-m$ agents with capability $c$ violate $\varphi$". Hence, we can get rid of all negations applied to tasks. %Hence, any CaTL+ formula can be represented in negation-free DNF.
\end{proof}

Let $sort(\cdot)$ reorder a sequence of scalars from largest to smallest and return the reordered index:
\vspace{-10pt}
\begin{equation}
    \label{eq:sort}
    sort([\eta_i]_{i=1}^N)=i_1,i_2,\ldots,i_N,\ s.t.\  \eta_{i_1}\geq\eta_{i_2}\geq\ldots\geq\eta_{i_N}.
\end{equation}
Next, we follow Alg.~\ref{alg:repair} to repair the output of CatlNet. We assume that an STL control synthesis algorithm is available. That is, given an STL formula $\varphi$ defined over an agent's individual trajectory, we can find the control sequence for the agent that steers it to satisfy the STL formula if a solution exists, denoted as $\mathbf x, \mathbf u\leftarrow syn(\varphi)$. We first rewrite the CaTL+ formula into its negation-free DNF $\Phi=\bigvee_{k=1}^K\bigwedge_{i=1}^{I_k}\bar T_{i}^k$. Then at the initial state, we predict the team trajectory using CatlNet and the system model. 
To satisfy $\Phi$, at least one of the $k$ clauses $\bigwedge_{i=1}^{I_k}\bar T_{i}^k$ needs to be satisfied. We calculate the robustness for all of them and consider these clauses from the highest robustness to the lowest (step 1). For the clause $\bigwedge_{i=1}^{I_k}\bar T_{i}^k$, if it is violated, we find all tasks $\bar T_i^k=\langle\varphi_i^k,c_i^k,m_i^k\rangle_{t_i^k}$ that are violated (or satisfied by exactly $m_i^k$ agents) and assign a STL formula $\mathbf F_{[t_i^k,t_i^k]}\varphi_i^k$ to enough ($m_i^k$) agents with the required capabilities $c_i^k$ (steps 2-6). %To make sure the other satisfied tasks will not be violated after the repair, we find all tasks $\bar T_i^k$ that are satisfied by exactly $m_i^k$ agents and assign $\varphi_i^k$ to them (steps 2-8).
Since we repair one agent at a time, those tasks satisfied by more than $m_i^k$ agents would not be violated. Now we have assigned a set of STL formulas to each agent that needs repair. Then we apply the STL control synthesis algorithm to make an agent's trajectory satisfy the conjunction of these formulas (step 8). After updating the trajectory, we find tasks satisfied by exactly $m_i^k$ agents and assign $\mathbf F_{[t_i^k,t_i^k]}\varphi_i^k$ to them again (steps 9-11). Repeat until all agents are repaired. If all agents get positive robustness, then the algorithm terminates and return \emph{success} (step 12). Otherwise redo all these steps for the next clause $\bigwedge_{i=1}^{I_k}\bar T_{i}^k$. If all clauses cannot be satisfied, then the algorithm terminates and return \emph{fail}. 

\RestyleAlgo{ruled}

\begin{algorithm2e}
\caption{Trajectory Repair}
\LinesNumbered
\label{alg:repair}
\KwIn{$\Phi=\bigvee_{k=1}^K\bigwedge_{i=1}^{I_k}\bar T_{i}^k$, $R_1=\cdots=R_{|J|}=\emptyset$, $F_1=\cdots=F_{|J|}=0$, $Result=fail$}
\For(\tcp*[f]{all clauses}){$k$ in $sort\big([\eta(\mathbf X,\bigwedge_{i=1}^{|I_k|}T_i^{k},0)]_{k=1}^{K}\big)$}{
\For(\tcp*[f]{all not (just) satisfied $\bar T_i^k$}){$\{i\in [1,I_k]\ |\ n(\mathbf X,c_i^{k},\varphi_i^k,t_i^k)\leq m_i^k)\}$}
{\For(\tcp*[f]{all agents with $c_i^k$}){$j^*$ in $sort\{\rho(x_j,\varphi_i^k,t_i^k)|j\in\mathcal J_{c_i^k}\}$}{$R_{j^*}\leftarrow R_{j^*}\cup\{i\}$;\tcp*[f]{assign task $i$ to agent $j^*$}\\ \lIf{$(\mathbf x_{j^*},t_i^k)\not\models\varphi_i^k$}{$F_{j^*}\leftarrow 1$; \tcp*[f]{flag agents that need repair} }
\lIf{task $T_i^k$ is assigned to $m_i^k$ agents}{break}}}(\tcp*[f]{assignment finished})
\For(\tcp*[f]{for all flagged agents}){$\{j\in\mathcal J|F_j=1\}$}{
$\mathbf x_j, \mathbf u_j\leftarrow syn(\bigwedge_{i\in R_j}\mathbf F_{[t_i^k,t_i^k]}\varphi_i^k)$; \tcp*[f]{get repaired trajectory}\\
\For{$\{i\in [1,I_k]\ |\ n(\mathbf X,c_i^{k},\varphi_i^k,t_i^k)= m_i^k)\}$}{\For{$j\in \mathcal J_{c_i^k}$}{
\lIf{$(\mathbf x_j,t_i^k)\models\varphi_i$}{$R_j\leftarrow R_j\cup \{i\}$ \tcp*[f]{redo assignment}}}}}
\lIf{all flagged agents get positive robustness}{$Result\leftarrow suc$, break}}
\Return $\mathbf X$, $\mathbf u_1, \cdots, \mathbf u_{|\mathcal J|}$, $Result$.
\end{algorithm2e}
\vspace{-5pt}

\begin{proposition}[soundness]
\label{prop:sound}
If Alg.~\ref{alg:repair} returns $Result=suc$, then the returned team trajectory $\mathbf X$ satisfies the CaTL+ specification: $(\mathbf X,0)\models\Phi$.
\end{proposition}

\begin{proof}
Consider a clause $\bigwedge_{i=1}^{I_k}\bar T_{i}^k$. To satisfy it, all timed tasks $\bar T_{i}^k$ should be satisfied. We first prove a lemma: by repeating steps 7-11, if the STL control synthesis returns positive robustness for all flagged agents ($Result=suc$ in step 12), then the returned team trjactory $(\mathbf X, 0)\models \bar T_{i}^k$, $\forall i\in \{1,\ldots, I_k\}$. Proof of this includes three steps: 

\begin{enumerate}
    \item For tasks $\bar T_{i}^k$ that are violated by the before-repaired team trajectory, i.e., $\mathbf n(\mathbf X,c_i^k,\varphi_i^k,t_i^k)<m_i^k$, they are assigned to $m_i^k$ agents. For these agents, if the original $(\mathbf x_j,t_i^k)\models\varphi_i^k$ and $F_j=0$, $\mathbf x_j$ will remain unchanged. If $(\mathbf x_j,t_i^k)\models\varphi_i^k$ and $F_j=1$, $\mathbf x_j$ will be updated but since $i\in R_j$, $(\mathbf x_j,t_i^k)\models \varphi_i^k$ (equivalent to $(\mathbf x_j,0)\models F_{[t_i^k,t_i^k]}\varphi_i^k$) holds after repair. If $(\mathbf x_j,t_i^k)\not\models\varphi_i^k$ and $F_j=1$, $\mathbf x_j$ will be repaired to satisfy $\varphi_i^k$. To sum up, all these $m_i^k$ agents will satisfy $\varphi_i^k$ after repair, i.e., $(\mathbf X,0)\models\bar T_{i}^k$.
    \item For tasks $\bar T_{i}^k$ that are satisfied by exactly $m_i^k$ agents before repair, i.e., $\mathbf n(\mathbf X,c_i^k,\varphi_i^k,t_i^k)=m_i^k$ they are assigned to $m_i^k$ agents. Similar as above, for these agents, no matter $F_j=1$ or $F_j=0$, they will satisfy $\varphi_i^k$ after repair, i.e., $(\mathbf X,0)\models\bar T_{i}^k$.
    \item For tasks $\bar T_{i}^k$ that are satisfied by more than $m_i^k$ agents before repair, i.e., $\mathbf n(\mathbf X,c_i^k,\varphi_i^k,t_i^k)>m_i^k$, they are not assigned to any agents. However, we repair one agent at a time. Hence, $\mathbf n(\mathbf X,c_i^k,\varphi_i^k,t_i^k)$ decreases at most $1$. If $\mathbf n(\mathbf X,c_i^k,\varphi_i^k,t_i^k)=m_i^k$ after the repair of an agent, we assign $\varphi_i^k$ to $m_i^k$ agents. Hence, the satisfaction of these tasks holds after repair. 
\end{enumerate}
In conclusion, all tasks $\bar T_{i}^k$ are satisfied, $\forall i\in \{1,\ldots, I_k\}$. The lemma is proved. 

According to step 12, if Alg.~\ref{alg:repair} returns $Result=suc$, then for one of the clauses all flagged agents get positive robustness. Using the lemma, one of the clauses $\bigwedge_{i=1}^{I_k}\bar T_{i}^k$ is satisfied. Hence, the disjunction of them is satisfied. 

\end{proof}

\vspace{-3pt}
\begin{remark}
Alg.~\ref{alg:repair} ensures soundness (Proposition \ref{prop:sound}) but it is not complete. Here completeness means that if a solution (a team trajectory that satisfies the CaTL+ specification) exists, then the algorithm can find it. In other words, it is possible that Alg.~\ref{alg:repair} returns $fail$ though a solution exists. 
\end{remark}

\vspace{-5pt}
\begin{remark}
Alg.~\ref{alg:repair} requires the global information (states of all agents), so it is only applied in the centralized training phase. Hence, the satisfaction of the CaTL+ specification in the execution is not guaranteed. However, simulation results show that by using the repair scheme as a guidance for the training, CatlNet can reach a high satisfying rate in the execution even without the repair. 
\end{remark}

\section{Case Studies}
Consider the scenario and CaTL+ specification in Ex.~\ref{ex}. Let the dimension of the communication vector be $8$, which contains the history information of each agent. Let $\mathcal U_j=[-1,1]^2$, $j=1,2,3,4$, $\mathcal U_j=[-1.2,1.2]^2$, $j=5,6$.  We train CaTL+ using the algorithm described in Sec.~\ref{sec:train}.  A team trajectory generated by the CatlNet after first training (full communication and no repair) is shown in Fig.~\ref{fig:results}(b). The trajectory violates the CaTL+ specification, mostly due to $\Phi_5$, i.e., only one ground vehicle can be on the bridge at a given time. To avoid appearing on the bridge at the same time, the agents tend to go across the bridge at the edge of the bridge, which is a local optimum with zero robustness. We use Alg.~\ref{alg:repair} to repair the trajectory in Fig.~\ref{fig:results}(b) which results in Fig.~\ref{fig:results}(c) with positive robustness. Then we retain CatlNet with the dataset and repeat the above process, which results in a CatlNet (full communication) that can generate satisfying team trajectories starting from random initial states. The final dataset contains $213$ satisfying trajectories.

\begin{figure}
    \centering
    \subfigure{
    \label{fig:map}
    \includegraphics[width=0.2\textwidth]{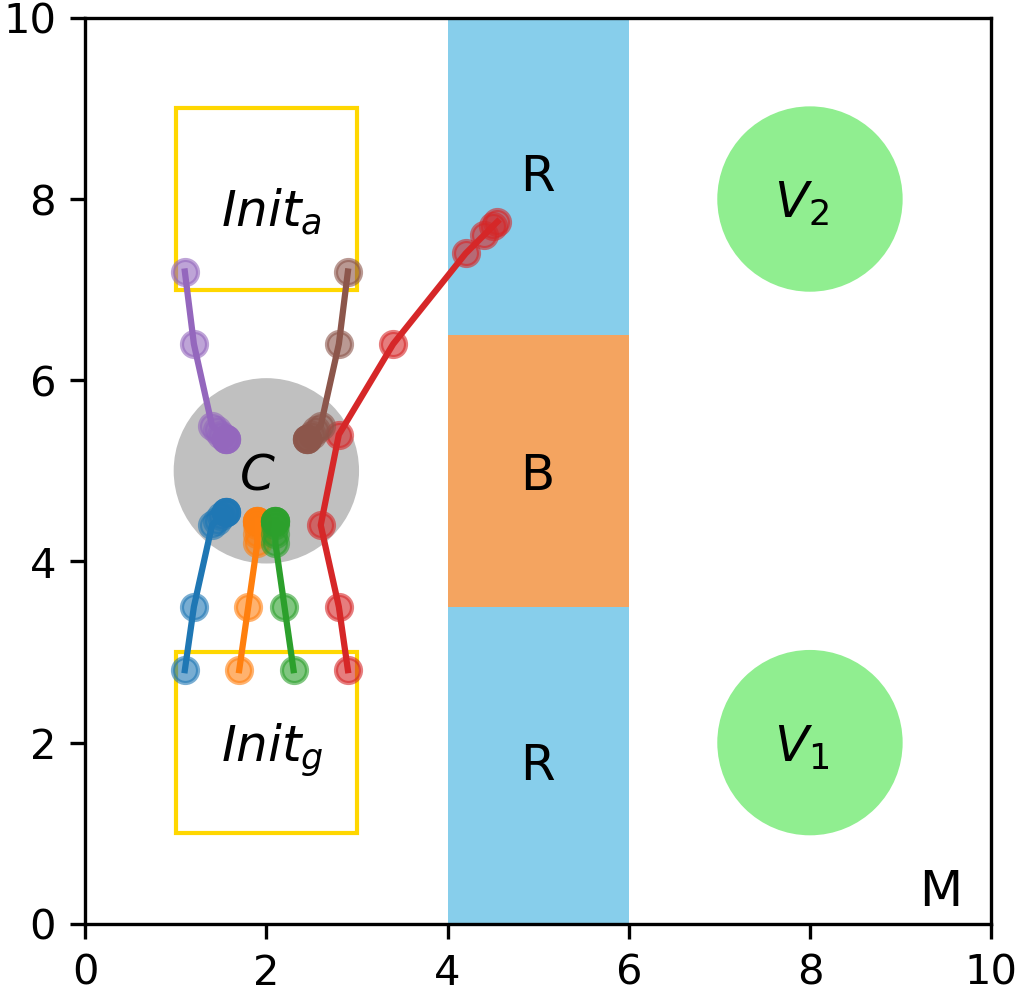}
    }
    \ 
    \subfigure{
    \label{fig:before_repair}
    \includegraphics[width=0.2\textwidth]{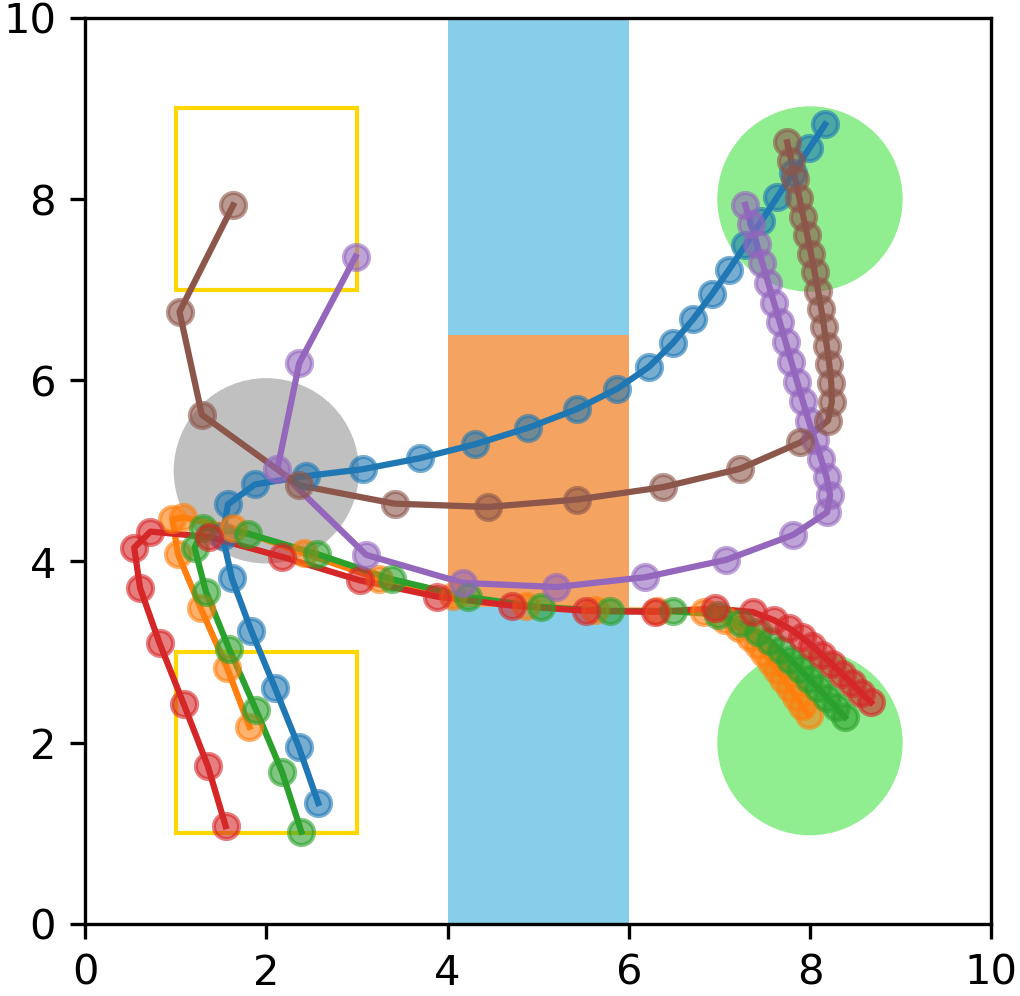}
    }
    \ 
    \subfigure{
    \label{fig:after_repair}
    \includegraphics[width=0.2\textwidth]{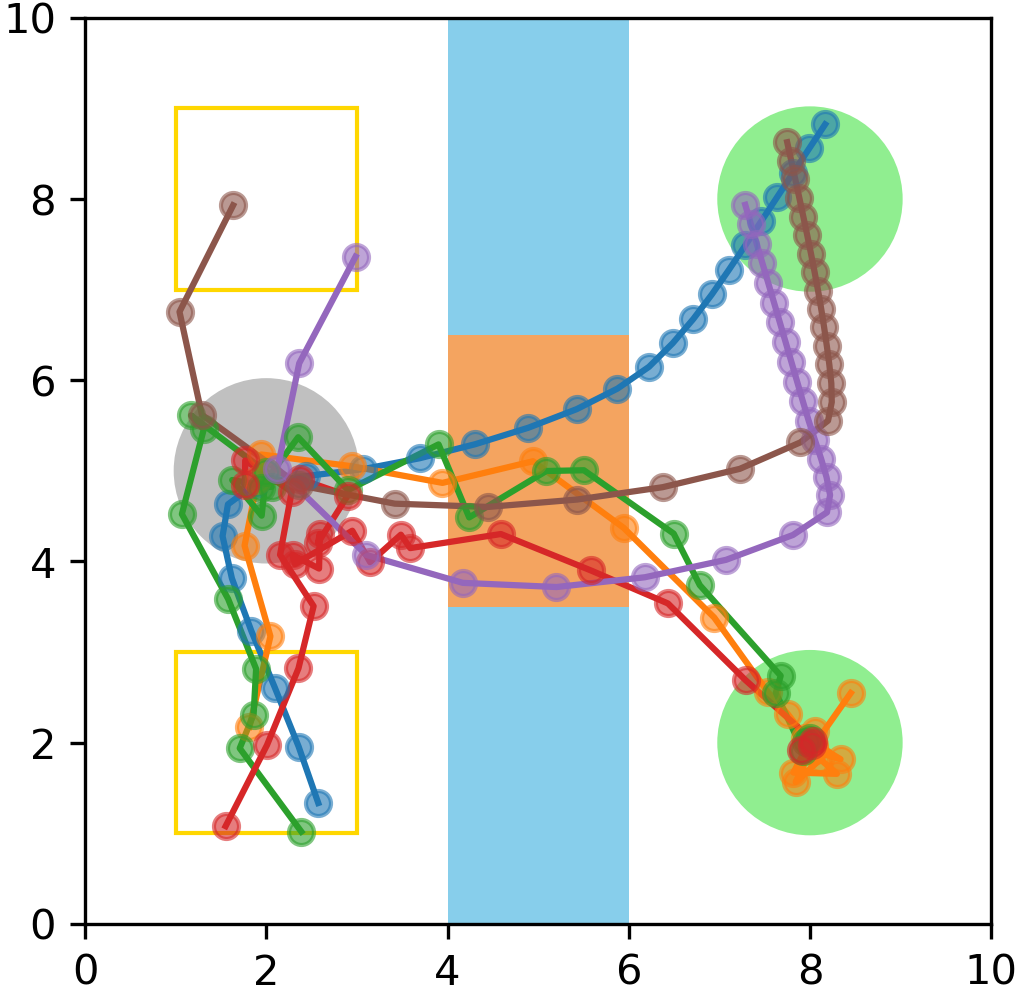}
    }
    \ 
    \subfigure{
    \label{fig:final}
    \includegraphics[width=0.2\textwidth]{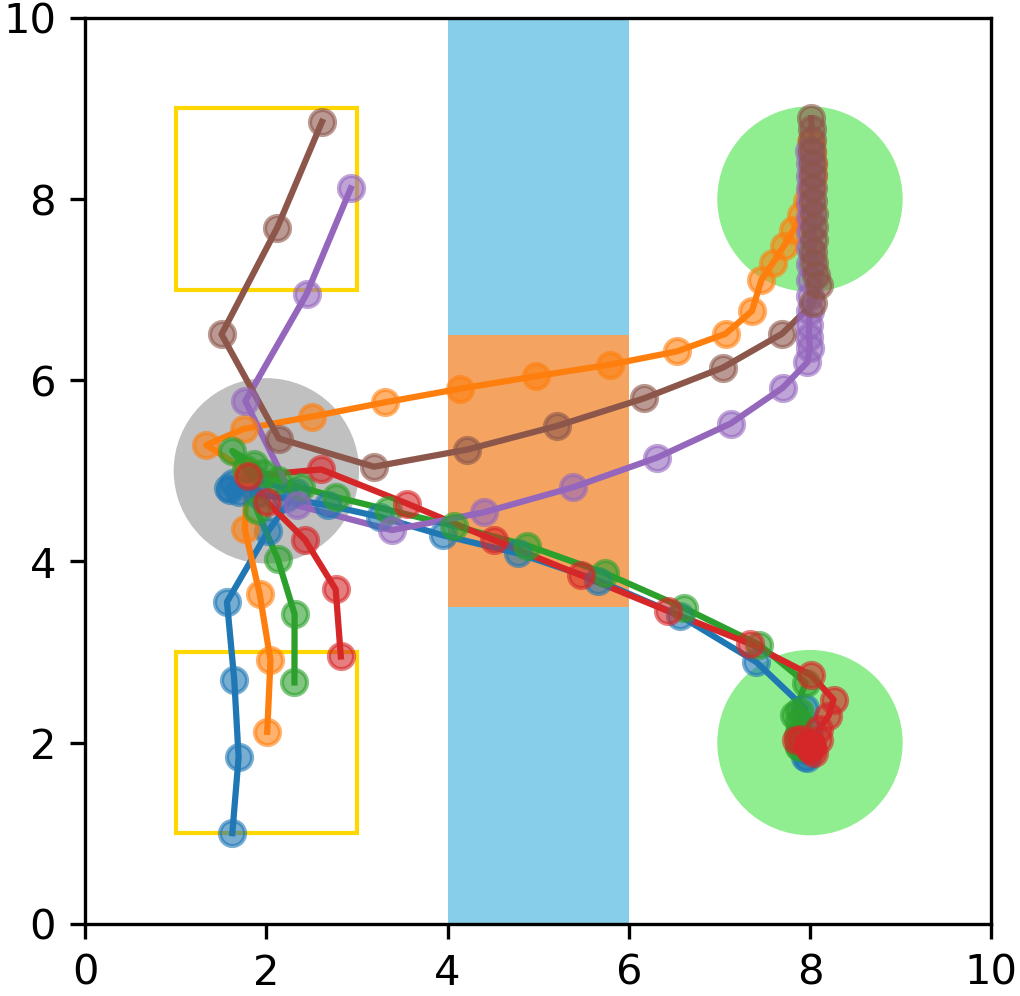}
    }
    \caption{\vspace{-4pt}\small (a) Environment and example trajectories. (b) Team trajectory before repair. (c) Team trajectory after repair. (d) Team trajectory generated by the final CatlNet.}
    \label{fig:results}
    \vspace{-8pt}
\end{figure}

\begin{figure}
\centering
\begin{minipage}[c]{0.45\textwidth}
    \centering
    \includegraphics[width=\textwidth]{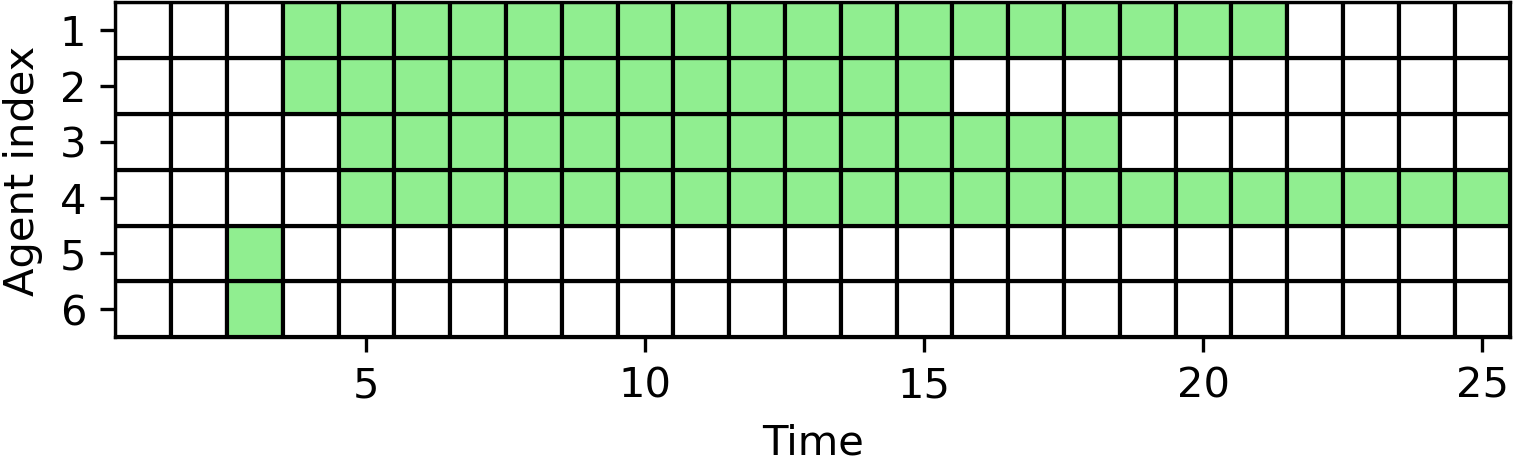}
\end{minipage}
\quad
\begin{minipage}[c]{0.48\textwidth}
    \centering
    \caption{\vspace{-20pt}\small Communication at each time step. Green squares mean the agent communicates at that time, while white squares indicate that the agent does not communicate. }
    \label{fig:comm}
\end{minipage}
\vspace{-8pt}
\end{figure}

Next, we train the Comm-gate and retrain the policy networks with it. A trajectory given by the final CatlNet is shown in Fig.~\ref{fig:results}(d). The corresponding communication at each time is shown in Fig.~\ref{fig:comm}. It can be seen that Comm-gate greatly reduces the total number of communications and the communication happens mainly when agents go across the river one by one. This makes sense as agents needs to behave differently at this stage of the task and communication enable them to do this. We test the final CatlNet from $10000$ random initial states and the success rate is $100.00\%$.

%\begin{figure}
%    \centering
%    \includegraphics[width=0.7\textwidth]{figures/comm.png}
%    \caption{Communication at each time step. Green means the agent communicate at that time, while white indicates the agent does not communicate.}
%    \label{fig:comm}
%\end{figure}

\begin{comment}
\begin{table}[]
    \centering
    \begin{tabular}{|c|c|c|c|c|c|c|}
        \hline
        Time &  0-3 & 4 & 5-11 & 12-15 & 16-20 & 21-24\\
        \hline
        Comm. & [0,0,0,0,0,0] &  [1,1,1,1,1,1] & [1,1,1,1,0,0] & [1,1,1,0,0,0] & [1,1,0,0,0,0] & [1,0,0,0,0,0]\\
        \hline
    \end{tabular}
    \caption{Communication at each time step.}
    \label{tab:comm}
\end{table}
\end{comment}

\section{Conclusion and Future Work}
We proposed a neural network-based model called CatlNet to learn both communication and distributed control policies from CaTL+ specifications. By using the repair algorithm during training, CatlNet can reach a high success rate of satisfying the specification. We plan to incorporate a lower level controller with CatlNet to avoid inter-agent collision and guarantee dense-time behaviors.

% Acknowledgments---Will not appear in anonymized version
\newpage
DISTRIBUTION STATEMENT A. Approved for public release. Distribution is unlimited. This material is based upon work supported by the Under Secretary of Defense for Research and Engineering under Air Force Contract No. FA8702-15-D-0001. Any opinions, findings, conclusions or recommendations expressed in this material are those of the author(s) and do not necessarily reflect the views of the Under Secretary of Defense for Research and Engineering. \copyright 2022 Massachusetts Institute of Technology. Delivered to the U.S. Government with Unlimited Rights, as defined in DFARS Part 252.227-7013 or 7014 (Feb 2014). Notwithstanding any copyright notice, U.S. Government rights in this work are defined by DFARS 252.227-7013 or DFARS 252.227-7014 as detailed above. Use of this work other than as specifically authorized by the U.S. Government may violate any copyrights that exist in this work.

\bibliography{references}

\end{document}